\documentclass[conference]{IEEEtran}
\usepackage{amssymb}
\usepackage{graphicx}
\usepackage{url}

\usepackage{float}
\usepackage{subfigure}
\usepackage{listings}
\graphicspath{{./Figs/}}
\usepackage{graphicx}
\usepackage{graphics}
\DeclareGraphicsExtensions{.jpg,.eps,.png}
\usepackage{amsmath}
\usepackage{amsfonts}
\usepackage{enumerate}
\usepackage{textcomp}
\usepackage{color}
\usepackage{algorithmic}
\usepackage{algorithm}
\usepackage{overpic}
\usepackage{amsthm}
\usepackage{bm}
\usepackage{dsfont}

\newtheorem{theorem}{Theorem}[section]

\newcommand{\RR}{\mathbb{R}}

 {\everymath{\displaystyle\everymath{}}\array}%
 {\endarray}
\everymath{\displaystyle\everymath{}}
\newcommand{\MYfooter}{\smash{
\hfil\parbox[t][\height][t]{.45\textwidth}{}\hfil\hbox{}}}
\makeatletter
\def\ps@IEEEtitlepagestyle{%
\def\@oddhead{\mbox{}2016 ICSEE International Conference on the Science of Electrical Engineering \rightmark \hfil }%
\def\@oddfoot{\MYfooter}%
\def\@evenfoot{\MYfooter}}
\makeatother
\begin{document}

\title{Randomized Independent Component Analysis}

\author{\IEEEauthorblockN{Matan Sela \qquad Ron Kimmel}
\IEEEauthorblockA{Department of Computer Science \\ Technion - Israel Institute of Technology}}

\maketitle

\begin{abstract}
Independent component analysis (ICA) is a method for recovering statistically independent 
 signals from observations of unknown linear combinations of the sources.
Some of the most accurate ICA decomposition methods require searching for the inverse transformation which minimizes different approximations of the Mutual Information, a measure of statistical independence of random vectors.
Two such approximations are the Kernel Generalized Variance or the Kernel Canonical Correlation which has been shown to reach the highest performance of ICA methods.
However, the computational effort necessary just for computing these measures is cubic 
 in the sample size. Hence, optimizing them becomes even more computationally demanding, 
 in terms of both space and time.
Here, we propose a couple of alternative novel measures based on randomized features of the samples - the Randomized Generalized Variance and the Randomized Canonical Correlation.
The computational complexity of calculating the proposed alternatives is linear in the 
 sample size and provide a controllable approximation of 
 their Kernel-based non-random versions.
We also show that optimization of the proposed statistical properties yields a comparable 
 separation error at an order of magnitude faster compared to Kernel-based measures.
\end{abstract}

\section{Introduction}
Independent component analysis (ICA) is a well-established problem in unsupervised learning 
 and signal processing, with numerous applications including blind source separation, 
 face recognition, and stock price prediction.
onsider the following scenario. A couple of speakers are located in a room. Each of them plays a different sound . Two microphones which are arbitrarily placed in the same room record unknown linear combinations of the sounds.  The goal of ICA is to process the signals recorded by the mics for recovering the soundtracks played by the speakers.

More precisely,
the basic idea of ICA is to recover $n_s$ statistically independent components of a non-Gaussian 
 random vector $s = \left( s_1 ,...,s_{n_s} \right)^T$ from  $n_s$ observed linear mixtures
 of its elements. 
That is, we assume that some unknown matrix $A \in \RR^{n_s \times n_s}$ mixes the entries of
 $s$ such that $x = As$. 
From samples of $x$, the goal is to estimate an un-mixing matrix $W$ such that $y = Wx$ 
 and the components of $y$ are statistically independent.

The matrix $W$ is found by a minimization process over a contrast function which measures
 the dependency between the unmixed elements. 
Ideally, finding the matrix $W$ which minimizes the {\it mutual information} (MI) provides the theoretically most accurate reconstruction.
The MI is defined as the Kullback-Liebler divergence between the joint distribution of $y$, $p(y_1,...,y_{n_s})$, and the product of its marginal distributions, $\prod_{i=1}^{n_s} p(y_i)$.
It is a non-negative function which vanishes if and only if the components of $y$
 are mutually independent.
Unfortunately, in practical applications, the joint and the marginal distributions
 are unknown. 
Estimating and optimizing the MI directly from the samples is difficult. 
Fitting a parametric or a nonparametric probabilistic model to the data based on which 
 MI is calculated is problem dependent and is often inaccurate. 

Many researches proposed alternative contrast functions.
One of the most robust alternative is the $\mathcal{F}$-correlation
 proposed in \cite{Bach:2003:KIC:944919.944920}. 
This function is evaluated by mapping the data samples into a
 {\it reproducing kernel Hilbert space}, $\mathcal{F}$ where a canonical correlation 
  analysis is performed. 
 The largest kernel canonical correlation (KCC) and the product of the kernel canonical
  correlations, known as the {\it kernel generalized variance} (KGV), are two possible 
  contrast functions. 
Despite their superior performance, algorithms based on minimizing the KCC or the KGV 
 (denoted as kernelized ICA algorithms) are less attractive for practical uses since the complexity of
  exact evaluation of these functions is cubic in the sample size.

A recent strand of research suggested randomized nonlinear feature maps for approximating 
 the reproducing kernel Hilbert space corresponding to kernel functions. 
This technique enables revealing nonlinear relations in a data by performing linear data
 analysis algorithms, such as Support Vector Machine \cite{NIPS2008_3495}, Principal Component
 Analysis and Canonical Correlation Analysis \cite{lopez2014randomized}.
These methods approximate the solution of kernel methods while reducing their complexity 
 from cubic to linear in the sample size.

Here, we propose two alternative contrast functions, Randomized Canonical Correlation (RCC) and Randomized
 Generalized Variance (RGV) which approximate KCC and KGV, respectively, yet require just 
 a fraction of the computational effort to evaluate.
Furthermore, the proposed random approximations are smooth, easy to optimize and converge 
 to their kernelized version as the number of random features grows. 
Finally, we propose optimization algorithms similar to those proposed in
 \cite{Bach:2003:KIC:944919.944920}, for solving the ICA problem. 
We demonstrate that our method has a comparable accuracy as KICA but runs 12 times 
 faster while separating components of real data.

\section{Background and Related Works}
\subsection{Canonical Correlation Analysis}
Canonical correlation analysis is a classical linear analysis method introduced 
 in \cite{hotelling1936relations}, which generalizes the Principal Component Analysis (PCA) 
 for two or more random vectors. 
In PCA, given samples of a random vector $x \in \RR^d$, the idea is to search for a vector 
 $w \in \RR^d$, which maximizes the variance of the projection of $x$ onto $w$. 
In practice, the principal components are the eigenvectors corresponding to the largest eigenvalues
 of the empirical covariance matrix $C = \frac{1}{N} X X^T$, where $X$ is a matrix containing
 samples of $x$ as its columns, and $N$ is the sample size.

In canonical correlation analysis (CCA), given a couple of random vectors $x \in \RR^{d_x}$ and 
 $y \in \RR^{d_y}$, one looks for a pair of vectors $w_x \in \RR^{d_x}$ and $w_y \in \RR^{d_y}$, 
 that maximize the correlation between the projection of $x$ onto $w_x$ and the projection 
 of $y$ onto $w_y$. 
More formally, CCA can be formulated as
\begin{align*}
 \max_{w_x \in \RR^{d_x},w_y \in \RR^{d_y}} \left| \text{corr}(w_x^Tx,w_y^Ty)\right| =  \cr =
  \frac{\left|\text{cov}(w_x^Tx,w_y^T y)\right|}{(\text{var}(w_x^T x))^{1/2}(\text{var}(w_y^T y))^{1/2}}
\end{align*}

Let $X \in \RR^{d_x \times N}$ and $Y \in \RR^{d_y \times N}$ be matrices of samples of $x$ 
and $y$, respectively.
The empirical Canonical Correlation Analysis problem is given by
\begin{align*}
\underset{W_x \in \RR^{d_x \times r},W_y \in \RR^{d_y \times r}}{\text{minimize}}
& & \| \widehat{\text{corr}}(W_x^T X,W_y^T Y) - I \| \cr
 \text{subject to}
&&  \widehat{\text{corr}}(W_x,W_x) = I, \cr
&& \widehat{\text{corr}}(W_y,W_y) = I,
\end{align*}
 where $\widehat{\text{corr}}( \cdot, \cdot)$ is the empirical correlation.
The canonical correlations $\lambda_1,...,\lambda_r$ are found by solving the following 
 generalized eigenvalue problem
\begin{align*}
\left( \begin{array}{cc} 0 & C_{XY} \\ C_{YX} & 0 \end{array} \right)
\left( \begin{array}{c} w_x \\ w_y \end{array} \right) = \cr 
\lambda \left( \begin{array}{cc} C_{XX}+\gamma I & 0 \\ 0 & C_{YY}+\gamma I  \end{array} \right)
\left( \begin{array}{c} w_x \\ w_y \end{array} \right), 
\end{align*}
where $C_{XY} = \frac{1}{N}XY^T$ is the empirical cross covariance matrix, and $\gamma I$ 
 is added to the diagonal for stabilizing the solution. 
As discussed next, the kernelized ICA method use a kernel formulation of  CCA for evaluating
 its contrast functions which measures the independence of random variables.
 
\subsection{Kernelized Independent Component Analysis}
In ICA, we minimize a contrast function which is defined as any non-negative function of two or more random variables that is zero if and only if they are statistically independent.
By definition, a pair of random variables $x_1$ and $x_2$ are said to be statistically 
 independent if $p(x_1,x_2) = p(x_1)p(x_2)$. 
As a result, for any two functions $f_1,f_2 \in \mathcal{F}$
\begin{align*}
E\{f_1(x_1)f_2(x_2)\} = E\{f_1(x_1)\} E\{f_2(x_2)\}
\end{align*}
Equivalently,
\begin{align*}
\text{cov}\left(f_1(x_1),f_2(x_2)\right) = 0.
\end{align*}
The $\mathcal{F}$-correlation $\rho_{\mathcal{F}}$ is defined as the maximal correlation among all the functions $f_1(x_1),f_2(x_2)$ in $\mathcal{F}$. That is
\begin{align*}
\rho_{\mathcal{F}} &=& \max_{f_1,f_2 \in \mathcal{F}} \left| \text{corr}\left( f_1(x_1),f_2(x_2) \right) \right|\cr &=& \max_{f_1,f_2 \in \mathcal{F}} \frac{\left| \text{cov}\left( f_1(x_1),f_2(x_2) \right)\right|}{\left( \text{var} f_1(x_1) \right)^{1/2}\left( \text{var} f_2(x_2) \right)^{1/2}}
\end{align*}
Obviously, if $x_1$ and $x_2$ are independent, then $\rho_{\mathcal{F}}=0$. 
As proven in \cite{Bach:2003:KIC:944919.944920},  if $\mathcal{F}$ is a functional space that 
 contains the Fourier basis ($f(x)=e^{i\omega x}, \omega \in \RR$), then, the opposite is also true. 
This property implies that $\rho_{\mathcal{F}}$ can replace the mutual information while searching
 for a matrix $W$ that transforms the vector $(x_1,x_2)^T$ into a vector with independent components.

Computing the $\mathcal{F}$-correlation directly for an arbitrary space $\mathcal{F}$ requires estimating
 the correlation between every possible pair of functions in the space, making the calculation impractical. However, if $\mathcal{F}$ is a reproducing kernel Hilbert space (RKHS), the $\mathcal{F}$-correlation can
 be evaluated. 
Let $k(x,y)$ be a kernel function associated with the inner product between functions in a reproducing
 kernel Hilbert space $\mathcal{F}$. 
Denote $\Phi(x)$ as the feature map corresponding to the kernel $k(x,y)$, such that 
 $k(x,y) = \langle \Phi(x) , \Phi(y) \rangle$.
The feature map is given by $\Phi(x) = k(\cdot,x)$.
Then, from the {\it reproducing property} of the kernel, for any function $f(x)\in \mathcal{F}$, 
\begin{align*}
f(x) = \langle \Phi(x),f \rangle, \qquad \forall f \in \mathcal{F}, \forall x \in \RR.
\end{align*}
It follows that for every functional space $\mathcal{F}$ associated with a kernel function $k(x,y)$ and a feature map $\Phi(x)$, the $\mathcal{F}$-correlation between a pair of random variables $x_1$ and $x_2$ can be formulated as
\begin{align*}
\rho_{\mathcal{F}} = \max_{f_1,f_2 \in \mathcal{F}} \left| \text{corr}\left( \langle 
    \Phi(x_1), f_1\rangle, \langle \Phi(x_2),f_2 \rangle \right) \right|.
\end{align*}

Let $\{x_1^i\}_{i=1}^N$ and $\{x_2^i\}_{i=1}^N$ be the samples of the random variables
 $x_1$ and $x_2$, respectively.
Any function $f \in \mathcal{F}$ can be represented by 
 $f = \sum_{i=1}^N \alpha^i \Phi(x^i) + f^{\bot}$, where $f^{\bot}$ is a function in the 
  subspace of functions orthogonal to $\text{span}\{\Phi(x^1),...,\Phi(x^N)\}$.
Thus,
\begin{align*}
&\widehat{\text{cov} }\left( \langle \Phi(x_1), f_1 \rangle, \langle \Phi(x_2),f_2 \rangle \right) =
\cr &= \frac{1}{N} \sum_{k=1}^N \langle \Phi(x_1^k) , f_1 \rangle \langle \Phi(x_2^k),f_2 \rangle 
 \cr &=\frac{1}{N} \sum_{k=1}^N \langle \Phi(x_1^k) , \sum_{i=1}^N \alpha_1^i \Phi(x_1^i) + f_1^{\bot} \rangle \langle \Phi(x_2^k),\sum_{i=1}^N \alpha_2^i \Phi(x_2^i) + f_2^{\bot} \rangle \cr
& = \frac{1}{N} \sum_{k=1}^N \langle \Phi(x_1^k) , \sum_{i=1}^N \alpha_1^i \Phi(x_1^i) \rangle \langle \Phi(x_2^k),\sum_{i=1}^N \alpha_2^i \Phi(x_2^i) \rangle  \cr
&= \frac{1}{N} \sum_{k=1}^N \sum_{i=1}^N \sum_{j=1}^N \alpha_1^i K_1(x_1^i,x_1^k) K_2(x_2^j,x_2^k)  \alpha_2^i
\cr &= \frac{1}{N} \alpha_1^T K_1 K_2 \alpha_2,
\end{align*}
where $K_1$ and $K_2$ are the empirical kernel Gram matrices of $x_1$ and $x_2$, respectively.
With a similar development for the empirical variance, one conclude
\begin{align*}
\widehat{\text{var}} \left( \langle \Phi(x_1),f_1 \rangle \right) = \cr \frac{1}{N} \alpha_1^T K_1 K_1 \alpha_1 \cr
\widehat{\text{var}} \left( \langle \Phi(x_2),f_2 \rangle \right) = \cr \frac{1}{N} \alpha_2^T K_2 K_2 \alpha_2 
\end{align*}

Therefore, the empirical $\mathcal{F}$-correlation between $x_1$ and $x_2$ is given by
\begin{align*}
\hat{\rho}_{\mathcal{F}} \left( x_1,x_2 \right) &= &
 \max_{\alpha_1,\alpha_2 \in \RR^N} \frac{\alpha_1^T K_1 K_2 \alpha_2}
  {\left( \alpha_1^T K_1^2 \alpha_1 \right)^{1/2} \left( \alpha_2^T K_2^2 \alpha_2 \right)^{1/2}}
\end{align*}
 which can be evaluated by solving the following generalized eigenvalue problem
\begin{align*}
\left( \begin{array}{cc} 0 & K_1 K_2 \\ K_2 K_1 & 0 \end{array} \right)
\left( \begin{array}{c} \alpha_1 \\ \alpha_2 \end{array} \right) = \cr 
\lambda \left( \begin{array}{cc} K_1^2 & 0 \\ 0 & K_2^2  \end{array} \right)
\left( \begin{array}{c} \alpha_1 \\ \alpha_2 \end{array} \right).
\end{align*}
The calculation of the $\mathcal{F}$-correlation is thus equivalent to solving a kernel CCA problem.
For ensuring computational stability, it is common to use a regularized version of KCCA given by
\begin{align*}
\left( \begin{array}{cc} 0 & K_1 K_2 \\ K_2 K_1 & 0 \end{array} \right)
\left( \begin{array}{c} \alpha_1 \\ \alpha_2 \end{array} \right) = \cr 
\lambda \left( \begin{array}{cc} (K_1+\frac{N\kappa}{2}I)^2 & 0 \\ 0 & (K_2+\frac{N\kappa}{2}I)^2  \end{array} \right)
\left( \begin{array}{c} \alpha_1 \\ \alpha_2 \end{array} \right).
\end{align*}

For $m$ random variables, the regularized KCCA amounts to finding the largest generalized eigenvalue of the following system
\begin{align*}&
\small{
\left( \begin{array}{cccc} 0 & K_1 K_2 & \cdots & K_1 K_m  \\
				     K_2 K_1 & 0  & \cdots & K_2 K_m \\
				     \vdots & \vdots & & \vdots \\
				     K_m K_1 & K_m K_2 & \cdots & 0 \end{array} \right)
\left( \begin{array}{c} \alpha_1 \\ \alpha_2 \\ \vdots \\ \alpha_m \end{array} \right)}  = \cr
&\small{
\lambda \left( \begin{array}{cccc} (K_1 + \frac{N \kappa}{2} I)^2 & 0 & \cdots & 0  \\
				     0 & (K_2 + \frac{N \kappa}{2} I)^2  & \cdots & 0 \\
				     \vdots & \vdots & & \vdots \\
				     0 & 0 & \cdots & (K_m + \frac{N \kappa}{2} I)^2 \end{array} \right)
\left( \begin{array}{c} \alpha_1 \\ \alpha_2 \\ \vdots \\ \alpha_m \end{array} \right)}
\end{align*}
 or in short, $\mathcal{K}_{\kappa} \alpha = \lambda \mathcal{D}_{\kappa} \alpha$.

This is the first contrast function proposed in \cite{Bach:2003:KIC:944919.944920}.
The second one, denoted as the kernel generalized variance, depends upon not only the largest 
 generalized eigenvalue of the problem above but also on the product of the entire spectrum.
This is derived from the Gaussian case where the mutual information is equal to minus half the
 logarithm of the product of generalized eigenvalues of the regular CCA problem.
In summary, the kernel generalized variance is given by
\begin{align*}
I(x_1,...,x_m) =  -\frac{1}{2} \sum_{i=1}^P \log{\lambda_i},
\end{align*}
 where $P$ is the rank of $\mathcal{K}_{\kappa}$.
Figure \ref{fig:kgvonly} demonstrates the Kernel Generalized Variance versus the Mutual Information 
 for a Gaussian kernel and for different $\sigma$'s. 
The complexity of a naive implementation of both KCC and KGV is $O(N^3)$. 
However, an approximate solution can be computed using Cholesky decomposition which reduces the 
 complexity to be roughly quadratic in the sample size.

\begin{figure}[ht]
\label{fig:kgvonly}
\includegraphics[width=.45\textwidth]{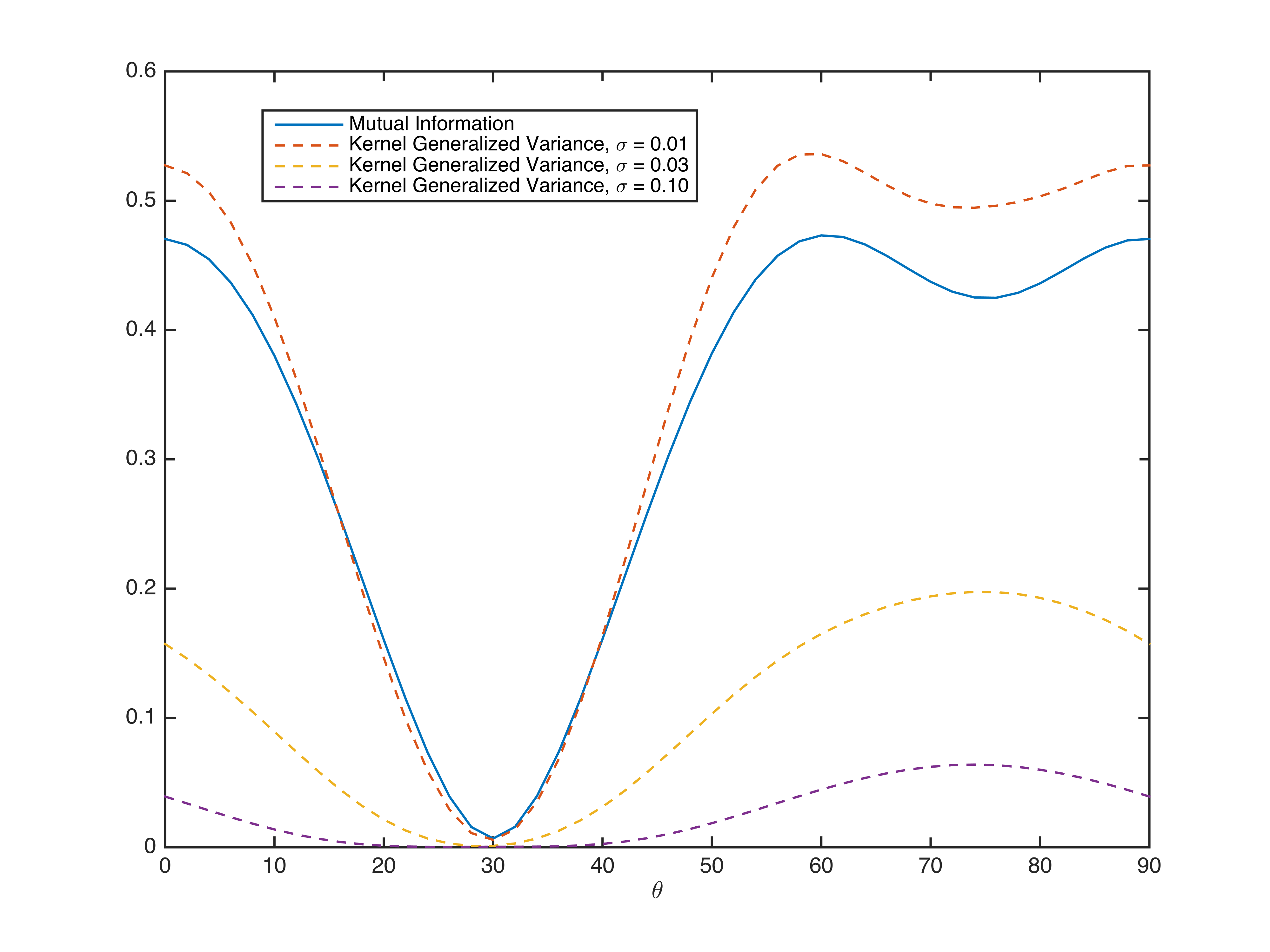}
	\caption{\small The mutual information and the Kernel Generalized Variance for different
     $\sigma$ values as a function of the angle of rotation of the orthogonal matrix $W$.
   }
	\label{fig:examples}
\end{figure}

\subsection{Randomized Features}
In kernel methods, each data sample $x$ is mapped to a function in the reproducing kernel Hilbert space $\Phi(x)$, where the analysis is performed. 
The inner product between two representational functions $\Phi(x)$ and $\Phi(y)$ can be evaluated 
 directly on the samples $x$ and $y$ using the kernel function $k(x,y)$. For real valued, 
 normalized ($k(x,y)\leq 1$), shift invariant kernels $\RR^d \times \RR^d$
\begin{align*}
&k(x,y)  =  \int_{\RR^d} p(w) e^{-jw^T(x-y)} dw  \approx \cr  
&\approx \sum_{i=1}^m \frac{1}{m} e^{-j w_i^T x} e^{j w_i^T y} \cr
&= \sum_{i=1}^m \frac{1}{m} \cos(w_i^Tx+b_i) \cos(w_i^Ty+b_i),
\end {align*}
 where $p(w)$ is the inverse Fourier transform of $k$ and $b_i \sim U (0,2\pi)$, 
 and $w_i$'s are independently drawn from  $p(w)$. 
Thus, the kernel function can be approximated by transforming the data samples into an $m$-dimensional 
 random space $z(x) = \frac{1}{\sqrt{m}} \left[ \cos(w_1^Tx + b_1) ,..., \cos(w_m^Tx + b_m) \right] $ 
 and taking the inner product between the maps.
\begin{align*}
 k(x,y) \approx \langle z(x), z(y) \rangle.
\end{align*}

The following theorem shows that the reproducing kernel Hilbert space corresponding to $k$ 
 can be approximated by a random map, in the inner product sense.
\begin{theorem}
Let $X \in\RR^{d \times N}$ be a matrix containing samples of $x$ ordered in its columns.
Denote $z(X) \in \RR^{m \times N}$ as a matrix containing the Fourier random features of
 each column of $X$ as its columns. Denote $\hat{K} = z(X)^Tz(x)$ and $K$ as the empirical 
  kernel matrix corresponding to the same kernel as $\hat{K}$. 
Then
\begin{align*}
\mathds{E} \| \hat{K} - K \| \leq \sqrt{\frac{3n^2\log{n}}{m}} + \frac{2n\log n}{m},
\end{align*}
 where the norm is an operator norm.
\end{theorem}
\begin{proof}
See \cite{lopez2014randomized}.
\end{proof}
Figure \ref{fig:kern_app} shows the analytic bound versus the empirical error of an approximate
 kernel $\hat{K}$ evaluated on $10^4$ data samples. 

\begin{figure}
\label{fig:kern_app}
\includegraphics[width=.45\textwidth]{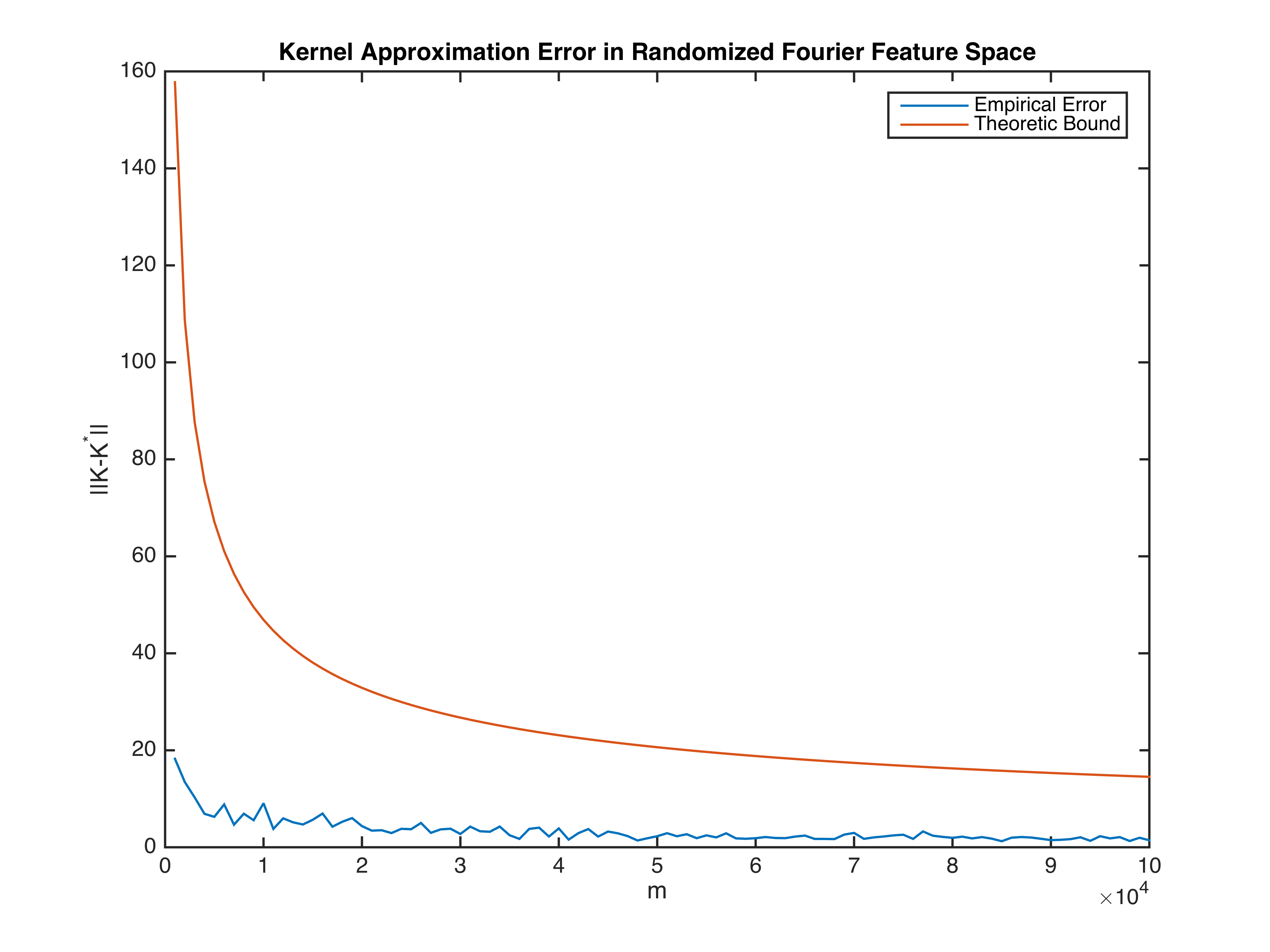}
\caption{\small The analytic and empirical error of the kernel approximation using random Fourier features as a function of the number of features $m$.}
\end{figure}

Introduced in $\cite{NIPS2008_3495}$ for approximating the solution of kernel Support Vector Machine,
 these random feature maps are also useful for solving kernel Principal Component Analysis (kPCA), 
 and kernel Canonical Correlation Analysis (kCCA), \cite{lopez2014randomized}.
The main purpose of using these features is the fact that they enable reducing the complexity
 of kernel methods to linear in the sample size, at the expense of a mild gap in accuracy. 
Here, we extend this idea for solving the problem of ICA.

\section{Randomized Independent Component Analysis}
As demonstrated in \cite{Bach:2003:KIC:944919.944920}, the kernel canonical correlation and the 
 kernel generalized variance measure the statistical dependence between a set of sampled random variables. 
In addition, optimization over these functions for decomposing mixtures of these variables 
 into independent ones is more accurate and robust. 
However, evaluating these functions requires $O(N^3)$ operations. 
Fortunately, for certain types of kernels, these contrast functions can be approximated using random
 features in linear time.

Let $k(x,y) = k(x-y)$ be a shift-invariant kernel function such that $k(x,y)\leq 1$. 
Denote $p(w)$ as the inverse Fourier transform of $k$, that is, 
 $k(x) = \int_{\RR^d} p(w) e^{-i2\pi w^Tx} dx$.
Independently sample $m$ variables $\{w_i\}_{i=1}^m$ from the distribution $p(w)$ and $m$ 
 random numbers $b_i$ uniformly from the section $\left[-\pi,\pi\right]$ and construct the r
  andom Fourier feature map 
   $z(x) = \sqrt{\frac{2}{m}}\left[ \cos(w_1^Tx + b_1) ,..., \cos(w_m^Tx + b_m) \right]$.
We define the  $\mathcal{Z}$-correlation of a pair of random variables $x_1$ and $x_2$ as
\begin{align*}
\rho_z (x_1,x_2) = \max_{w_1 \in \RR^{m} , w_2 \in \RR^{m}} \left| \text{corr} (w_1^Tz(x_1), w_2^T z(x_2)) \right|
\end{align*}

The empirical covariance between $w_1^Tz(x_1)$ and $w_2^Tz(x_2)$ is given by
 \begin{align*}
 &\widehat{\text{cov}}(w_1^Tz(x_1),w_2^Tz(x_2)) = \cr 
 &=\frac{1}{N} \sum_{k=1}^N \langle w_1, z(x_1^k) \rangle \langle w_2, z(x_2^k) \rangle \cr
 &= \frac{1}{N} \sum_{k=1}^N  w_1^T z(x_1^k)  z(x_2^k)^Tw_2  \cr
& =   w_1^T \left( \frac{1}{N} \sum_{k=1}^N z(x_1^k)  z(x_2^k)^T \right)w_2  \cr
&=   w_1^T \hat{C}_{12}^z w_2  
   \end{align*}
Similarly, the empircal variances can be computed as
 \begin{align*}
 \widehat{var}(w_1^Tz(x_1)) = w_1^T \hat{C}_{11}^z w_1  \cr
 \widehat{var}(w_2^Tz(x_2)) = w_1^T \hat{C}_{22}^z w_2,
   \end{align*}
where $\hat{C}_{11}^z = \frac{1}{N} \sum_{k=1}^N z(x_1^k)  z(x_1^k)^T$ 
 and $\hat{C}_{22}^z = \frac{1}{N} \sum_{k=1}^N z(x_2^k)  z(x_2^k)^T$ .

Thus the empirical $\mathcal{Z}$-correlation is the largest generalized eigenvalue of 
 the system
\begin{eqnarray}
\left( \begin{array}{cc} 0 & \hat{C}_{12}^z \\ \hat{C}_{21}^z & 0 \end{array} \right)
\left( \begin{array}{c} w_1 \\ w_2 \end{array} \right)= \cr
\lambda \left( \begin{array}{cc} \hat{C}_{11}^z + \gamma I & 0 \\ 0 & \hat{C}_{22}^z + \gamma I \end{array} \right)
\left( \begin{array}{c} w_1 \\ w_2 \end{array} \right).
\label{eqn-rcca}
\end{eqnarray}
As demonstrated in Figure \ref{fig:kcca_vs_rcca}, as the number of random features grows, 
 the $\mathcal{Z}$-correlation converges to the $\mathcal{F}$-correlation.
\begin{figure}[ht]
\includegraphics[width=.45 \textwidth]{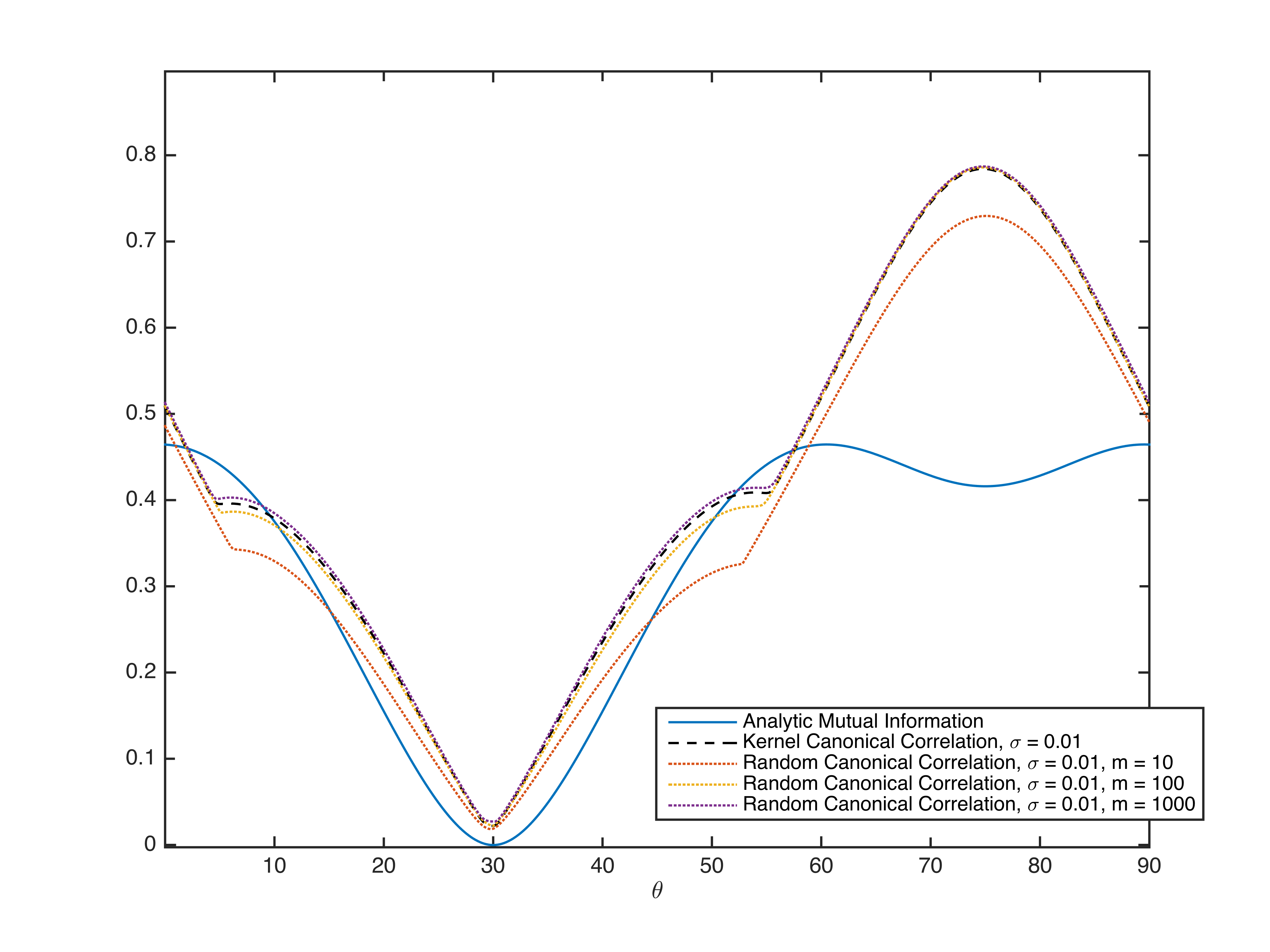}
\caption{\small Approximation of the Kernel Canonical Correlation using Randomized Canonical 
     Correlation with different numbers of features.
   }	\label{fig:kcca_vs_rcca}
\end{figure}
Notice that the size of the generalized eigenvalue system in the random feature space is  
 $n_s m \times n_s m$ which is much smaller than the kernel case where it is of size
 $n_s N \times n_s N$.

The kernel canonical correlation provide less accurate results than the kernel generalized variance 
 since it takes into account only the largest generalized eigenvalue. 
This can also be justified by the Gaussian case where the kernel generalized variance is shown to
 be a second order approximation of the mutual information around the point of independence 
  as $\sigma$ goes to zero.
Thus, we propose to approximate the kernel generalized variance by 
 \begin{align*}
 \delta_z(x_1,...,x_{n_s}) = -\frac{1}{2}\sum_{k=1}^P \log{\lambda_k}
 \end{align*}
where $\lambda_k$ are the generalized eigenvalues of \ref{eqn-rcca}. We define this function
 as the random generalized variance (RGV). 
As demonstrated in Figure \ref{fig:kgv_vs_rgv}, the more features we take for calculating
 the empirical covariances in the random space, the better RGV approximates KGV.
\begin{figure}[ht]
\includegraphics[width=.45\textwidth]{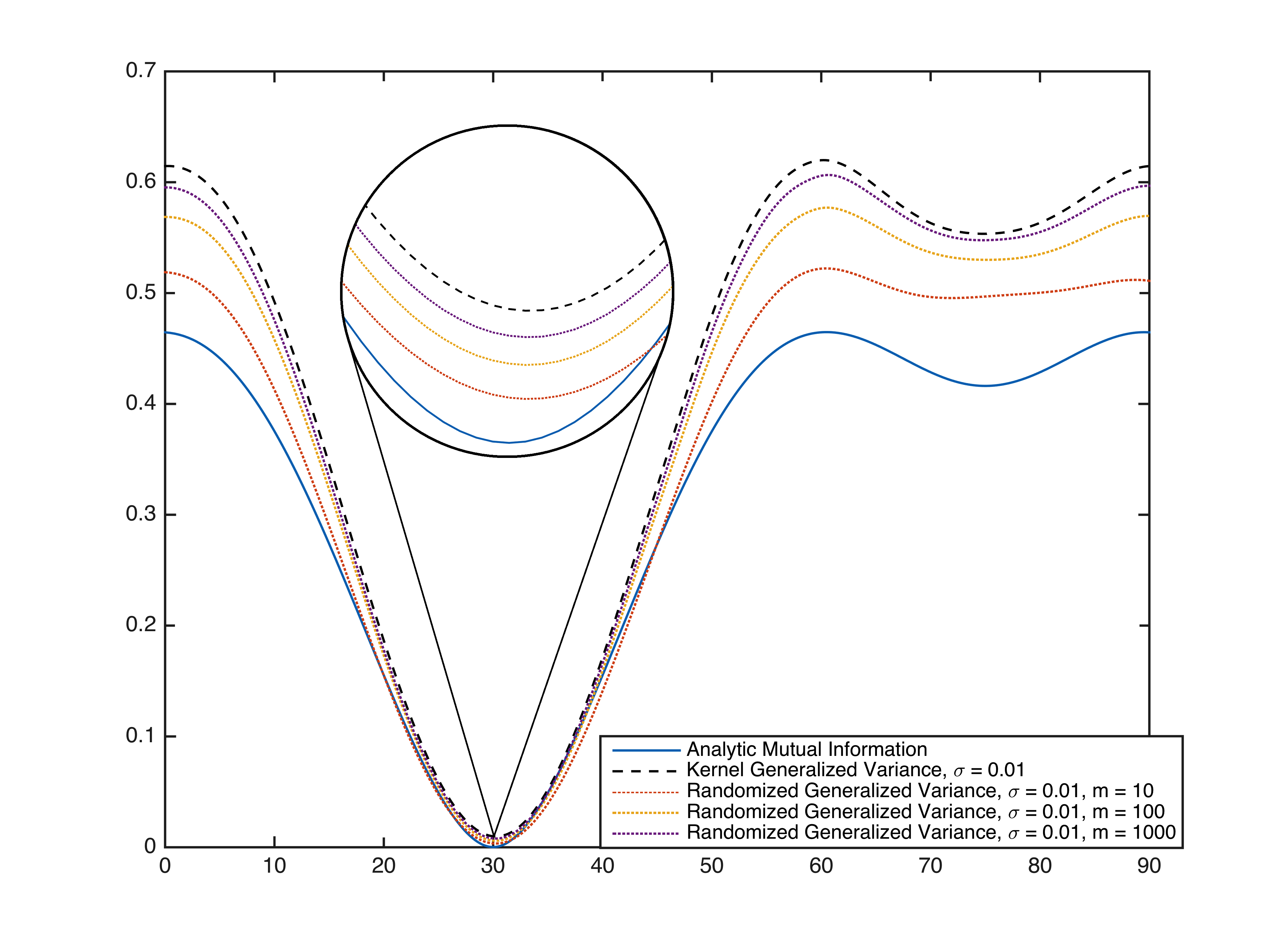}
\caption{\small Approximation of the Kernel Generalized Variance using Randomized Generalized 
     Variance with different numbers of features.
   }
	\label{fig:kgv_vs_rgv}
\end{figure}

\section{Experimental Results}
To evaluate the performance of optimization over our novel contrast functions for solving ICA problems, we independently draw samples from two or more probability functions given in Figure \ref{fig:pdfs}. 
Then, we applied some random transformation with condition number between one to two. 
For measuring and comparing the accuracy of our algorithm in estimating the unmixing matrix from the
 mixed data, we used the Amari distance given by
\begin{align*}
d(V,W) = \frac{1}{2n_s} \sum_{i=1}^{n_s} \left( \frac{\sum_{j=1}^{n_s} |a_{ij}| }
 {\max_j |a_{ij}} - 1 \right) + \cr+\frac{1}{2 n_s} \sum_{j=1}^{n_s}
\left( \frac{\sum_{i=1}^{n_s} |a_{ij}|}{\max_i |a_{ij}|} - 1 \right),
\end{align*} 
where $a_{ij} = (VW^{-1})_{ij}$. This metric, introduced in \cite{amari1996new},
 is invariant to scaling and permutation of the rows or columns of the matrices.
\begin{figure}[ht]
\includegraphics[width=.45 \textwidth]{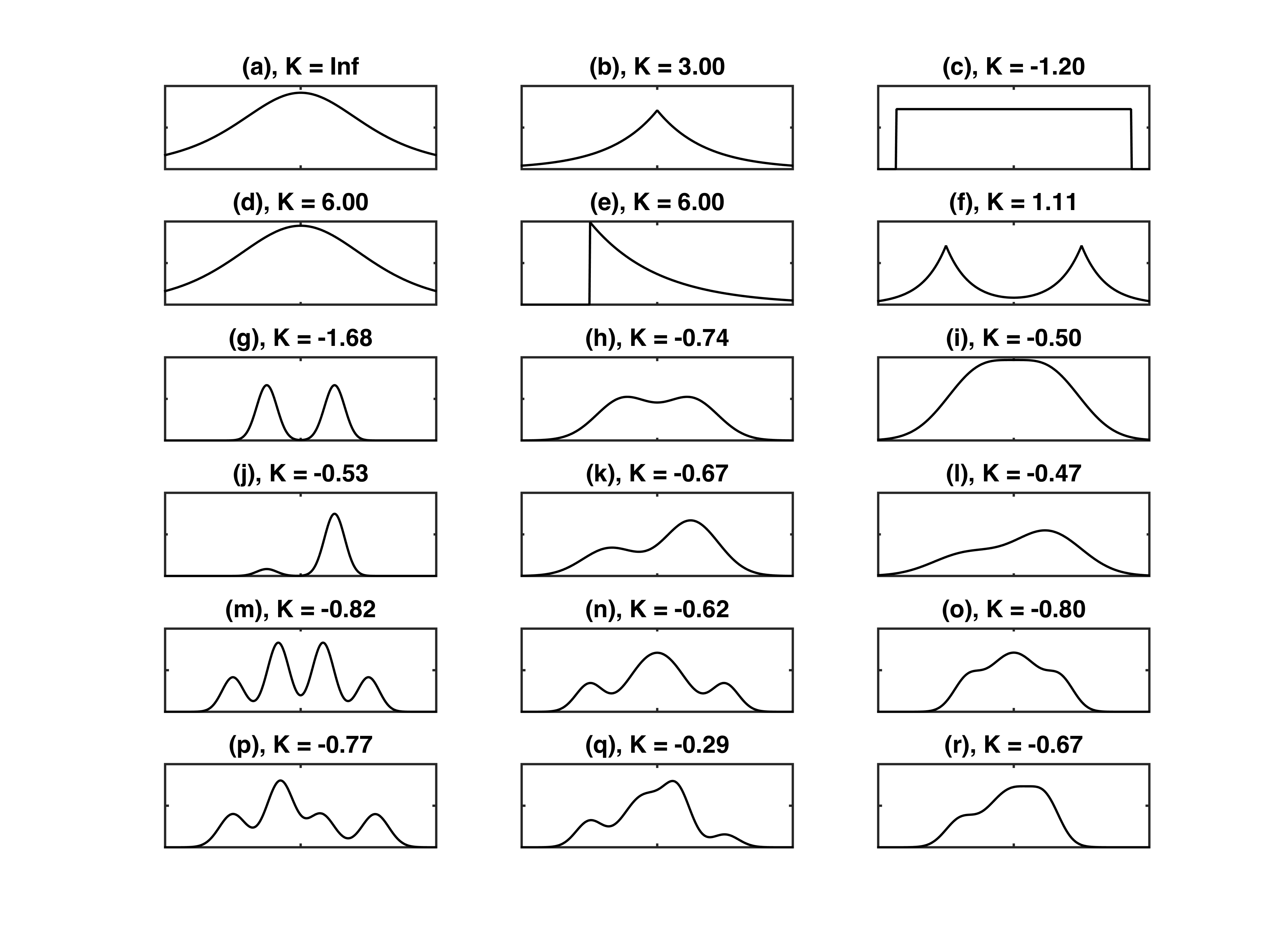}
	\caption{\small Probability density functions used in our tests. 
     These pdf-s are identical to those used in \cite{Bach:2003:KIC:944919.944920}.}
	\label{fig:pdfs}
\end{figure}

First, we evaluated the performance of our randomized ICA approach (RICA) on two mixtures of
 independent sources drawn in independently identically distributed fashion from the pdfs in \ref{fig:pdfs}. We repeated the experiments for 250 and 1000 samples, and the Amari distance for various ICA algorithms
  are summerized in Tables \ref{tbl:pairwise250} and \ref{tbl:pairwise1000}, respectively. 
As expected, the our results are comparable to those of KICA \cite{Bach:2003:KIC:944919.944920}, 
 while our algorithm is strictly linear in the sample size.
\begin{figure}[ht]
\includegraphics[width=.45\textwidth]{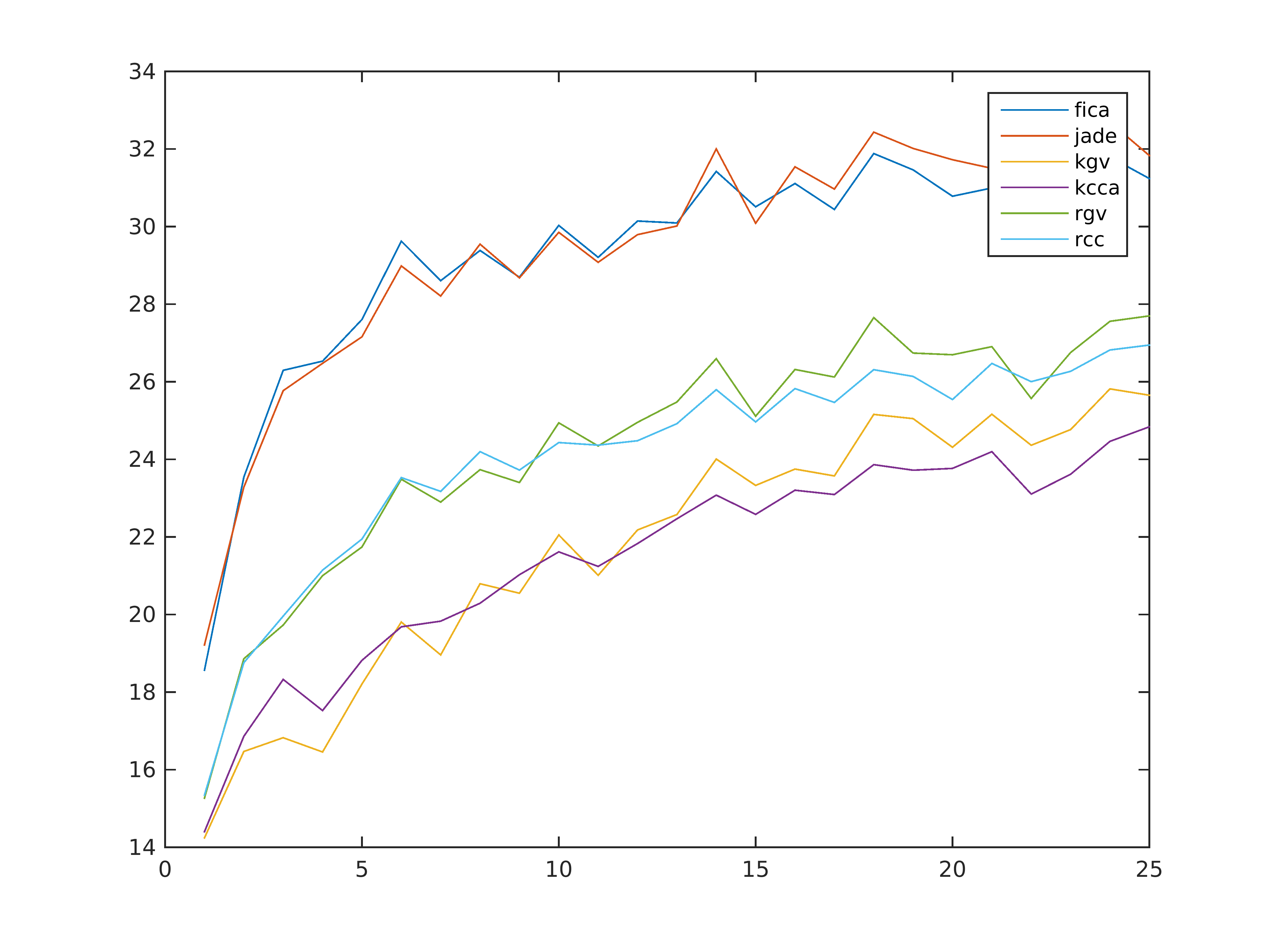}
	\caption{\small Robustness to outliers.}
	\label{fig:outliers}
\end{figure}

\begin{table}[h]
\renewcommand{\arraystretch}{1.1}
\centering
\caption {The Amari errors (multiplied by 100) of state-of-the-art ICA algorithms evaluated on a
 couple of mixtures of independent sources drawn from the denoted pdfs. 
In this experiment we used 250 samples and repeated the experiment 1000 times. 
The KGV, KCC, RCC, and RGV algorithms were initialized with the similar matrices. 
Each row represent the pdfs from which the samples were drawn. 
The mean row represent the average performance of each method and rand row denotes the performance 
 of the algorithms for a mixture of a pairs of sources randomly selected from the pdfs. 
 We repeated this experiment 1000 times.}
\scriptsize
\begin{tabular}{| c | c  c  | c  c  | c c |} 
\hline
pdfs  & F-ica & Jade & KCC & \textbf{RCC}  &  KGV &  \textbf{RGV}  \\
\hline  
a & 8.1 & 7.2 & 9.6 & 8.8 & 7.3 & \textbf{6.9} \\
b & 12.5 & 10.1 & 12.2 & 10.7 & 9.6 & \textbf{8.7} \\
c & 4.9 & 3.6 & 4.9 & 4.5 & \textbf{3.3} & 3.6 \\
\hline
d & 12.9 & \textbf{11.2} & 15.4 & 13.6 & 12.8 & 12.0 \\
e & 10.6 & 8.6 & 3.8 & 3.7 & \textbf{2.9} & 3.1 \\
f & 7.4 & 5.2 & 3.9 & 4.2 & \textbf{3.2} & 3.5 \\
\hline
g & 3.6 & 2.8 & 3.0 & 2.9 & 2.7 & \textbf{2.7} \\
h & 12.4 & \textbf{8.6} & 17.6 & 14.4 & 14.4 & 12.2 \\
i & 20.9 & \textbf{16.5} & 33.7 & 28.9 & 31.1 & 29.0 \\
\hline
j & 16.3 & 14.2 & 3.4 & 3.2 & 2.9 & \textbf{2.9} \\
k & 13.2 & 9.7 & 9.3 & 8.1 & 7.0 & \textbf{6.5} \\
l & 21.9 & 18.0 & 18.1 & 14.4 & 15.5 & \textbf{13.5} \\
\hline
m & 8.4 & 5.9 & 4.3 & 11.1 & \textbf{3.2} & 7.2 \\
n & 12.9 & 9.7 & 8.4 & 11.4 & \textbf{4.7} & 7.5 \\
o & 9.8 & \textbf{6.8} & 15.5 & 14.5 & 10.9 & 10.8 \\
\hline
p & 9.1 & 6.4 & 4.9 & 6.0 & \textbf{3.4} & 4.7 \\
q & 33.4 & 31.4 & 13.0 & 16.5 & \textbf{8.5} & 11.9 \\
r & 13.0 & 9.3 & 13.5 & 11.7 & 9.7 & \textbf{9.2} \\
\hline
\textbf{mean} & \textbf{12.8} & \textbf{10.3} & \textbf{10.8} & \textbf{10.5} & \textbf{8.5} & \textbf{8.7} \\
\hline
\textbf{rand} & \textbf{10.5} & \textbf{9.0} & \textbf{8.0} & \textbf{8.7}  & \textbf{5.9} & \textbf{6.8}\\
  \hline
\end{tabular}
\label{tbl:pairwise250}
\end{table}

\begin{table}[h]
\renewcommand{\arraystretch}{1.1}
\centering
\caption {The Amari errors (multiplied by 100) and the runtime }
\scriptsize
\begin{tabular}{| c | c   c | c  c c c | c c c  c |} 
\hline
pdfs  & F-ica & Jade  & KCCA & \textbf{RCC}&  KGV &  \textbf{RGV}   \\
\hline  
a  & 4.2 & 3.6 & 5.3 & 4.9 & \textbf{3.0} & 3.5 \\
b & 5.9 & 4.7 & 5.4 & 5.3 & \textbf{3.0} & 3.9 \\
c & 2.2 & 1.6 & 2.0 & 1.7 & 1.4 & \textbf{1.4} \\
\hline
d & 6.8 & 5.3 & 8.4 & 7.4 & \textbf{5.3} & 5.9 \\
e & 5.2 & 4.0 & 1.6 & 1.6 & \textbf{1.2} & 1.3 \\
f & 3.7 & 2.6 & 1.7 & 1.9 & \textbf{1.3} & 1.4 \\
\hline
g & 1.7 & 1.3 & 1.4 & 1.2 & 1.2 & \textbf{1.1} \\
h & 5.4 & \textbf{3.9} & 6.5 & 6.0 & 4.4 & 4.4 \\
i & 9.0 & \textbf{6.6} & 12.9 & 12.1 & 10.6 & 10.0 \\
\hline
j & 5.8 & 4.2 & 1.5 & 1.3 & 1.3 & \textbf{1.2} \\
k & 6.3 & 4.4 & 4.3 & 3.6 & \textbf{2.6} & 2.7 \\
l & 9.4 & 6.7 & 7.0 & 6.6 & 4.9 & \textbf{4.9} \\
\hline
m & 3.7 & 2.6 & 1.7 & 3.1 & \textbf{1.3} & 2.0 \\
n & 5.3 & 3.7 & 2.8 & 3.2 & \textbf{1.8} & 2.3 \\
o & 4.2 & \textbf{3.1} & 5.2 & 4.8 & 3.3 & 3.4 \\
\hline
p & 4.0 & 2.7 & 2.0 & 2.4 & \textbf{1.4} & 1.7 \\
q & 16.4 & 12.4 & 3.8 & 4.5 & \textbf{2.1} & 3.0 \\
r & 5.7 & 4.0 & 5.3 & 4.5 & \textbf{3.2} & 3.3 \\
\hline
\textbf{mean} & \textbf{5.8} & \textbf{4.3} & \textbf{4.4} & \textbf{4.2} & \textbf{3.0} & \textbf{3.2} \\
\hline
\textbf{rand} & \textbf{5.8} & \textbf{4.6} & \textbf{3.6} & \textbf{3.7} & \textbf{2.5} & \textbf{2.8} \\
 \hline
\end{tabular}
\label{tbl:pairwise1000}
\end{table}

\begin{table}[h]
\renewcommand{\arraystretch}{1.1}
\centering
\caption {The Amari errors (multiplied by 100) and the runtime analysis in separating a pair of independent audio signals from two recorded mixtures. }
\begin{tabular}{| c | c | c | c | c | c ||} 
\hline  
Method  & $\#$ repl & Amari Distance & Runtime (seconds)  \\
\hline 
KCC &  100 & 3.2 & 337.7 \\
\textbf{RCC} &  100 & 3.3 & 28.4 \\
\hline 
KGV &  100 & 1.2 & 258.0 \\
\textbf{RGV} & 100 & 1.3 & 23.6 \\
\hline 
\end{tabular}
\label{tbl:real_data}
\end{table}

We compared the robustness of our contrast function to outlier samples. 
The evaluation was done by choosing at random a certain number of points in the dataset and adding $5$ or $-5$ at random with probability $0.5$. 
The Amari distance between the estimated matrix and the true one for several algorithms is demonstrated in Figure \ref{fig:outliers}. We repeated the experiments 1000 times and averaged the results.
The most robust algorithms are KICA \cite{Bach:2003:KIC:944919.944920}, while RICA algorithms are still more robust then the others.

Since our framework imitates the kernel ICA methods in the random feature space, we evaluated the accuracy and the runtime of the methods in unmixing real audio signal. The results are given in Table \ref{tbl:real_data}.
With comparable accuracy, our methods run more than ten times faster.

\section{Conclusions}
We proposed two novel pseudo-contrast functions for solving the ICA problem.
The functions are evaluated using randomized Fourier features and thus can be computed in linear time.
As the number of feature grows, the proposed functions converge to the renowned  kernel generalized
 variance  the kernel canonical correlation which require a computational effort that is cubic
 in the sample size.
The accuracy of the proposed ICA methods is comparable to the state-of-the-art but runs over
 ten times faster.
The proposed functions could also be evaluated using Nystrom extension for kernel matrices.

\bibliographystyle{IEEEtran}
\bibliography{refs.bib}

\begin{thebibliography}{1}
\providecommand{\url}[1]{#1}
\csname url@samestyle\endcsname
\providecommand{\newblock}{\relax}
\providecommand{\bibinfo}[2]{#2}
\providecommand{\BIBentrySTDinterwordspacing}{\spaceskip=0pt\relax}
\providecommand{\BIBentryALTinterwordstretchfactor}{4}
\providecommand{\BIBentryALTinterwordspacing}{\spaceskip=\fontdimen2\font plus
\BIBentryALTinterwordstretchfactor\fontdimen3\font minus
  \fontdimen4\font\relax}
\providecommand{\BIBforeignlanguage}[2]{{%
\expandafter\ifx\csname l@#1\endcsname\relax
\typeout{** WARNING: IEEEtran.bst: No hyphenation pattern has been}%
\typeout{** loaded for the language `#1'. Using the pattern for}%
\typeout{** the default language instead.}%
\else
\language=\csname l@#1\endcsname
\fi
#2}}
\providecommand{\BIBdecl}{\relax}
\BIBdecl

\bibitem{Bach:2003:KIC:944919.944920}
\BIBentryALTinterwordspacing
F.~R. Bach and M.~I. Jordan, ``Kernel independent component analysis,''
  \emph{J. Mach. Learn. Res.}, vol.~3, pp. 1--48, Mar. 2003. [Online].
  Available: \url{http://dx.doi.org/10.1162/153244303768966085}
\BIBentrySTDinterwordspacing

\bibitem{NIPS2008_3495}
\BIBentryALTinterwordspacing
A.~Rahimi and B.~Recht, ``Weighted sums of random kitchen sinks: Replacing
  minimization with randomization in learning,'' in \emph{Advances in Neural
  Information Processing Systems 21}, D.~Koller, D.~Schuurmans, Y.~Bengio, and
  L.~Bottou, Eds.\hskip 1em plus 0.5em minus 0.4em\relax Curran Associates,
  Inc., 2009, pp. 1313--1320. [Online]. Available:
  \url{http://papers.nips.cc/paper/3495-weighted-sums-of-random-kitchen-sinks-replacing-minimization-with-randomization-in-learning.pdf}
\BIBentrySTDinterwordspacing

\bibitem{lopez2014randomized}
D.~Lopez-Paz, S.~Sra, A.~Smola, Z.~Ghahramani, and B.~Sch{\"o}lkopf,
  ``Randomized nonlinear component analysis,'' \emph{arXiv preprint
  arXiv:1402.0119}, 2014.

\bibitem{hotelling1936relations}
H.~Hotelling, ``Relations between two sets of variates,'' \emph{Biometrika},
  vol.~28, no. 3/4, pp. 321--377, 1936.

\bibitem{amari1996new}
S.-i. Amari, A.~Cichocki, H.~H. Yang \emph{et~al.}, ``A new learning algorithm
  for blind signal separation,'' \emph{Advances in neural information
  processing systems}, pp. 757--763, 1996.

\end{thebibliography}

\end{document}